\documentclass{cys}
\usepackage[ansinew]{inputenc}
\usepackage{graphicx}

\usepackage{microtype}
\usepackage{subfigure}
\usepackage{booktabs} 
\usepackage{amsmath,amssymb,amsfonts,amsthm}
\usepackage{textcomp}
\usepackage{multirow}
\usepackage{color}
\usepackage{subfigure} 
\newtheorem{theorem}{Theorem}
\newtheorem{lemma}{Lemma}
\newtheorem{definition}{Definition}

\usepackage{hyperref}
\hypersetup{draft}

\usepackage{algorithm}
\usepackage{algorithmic}

\title{Wrapped Loss Function for Regularizing Nonconforming Residual Distributions}

\author{Chun Ting Liu$^1$, Ming Chuan Yang$^2$, and Meng Chang Chen$^3$}

\affil{ 
Information Processing and Discovery Lab, Institute of Information Science, Academia Sinica, Taiwan \authorcr

\authorcr  \authorcr
jimliu741523@gmail.com, mingchuan@iis.sinica.edu.tw,mcc@iis.sinica.edu.tw
\authorcr  \authorcr
}

\begin{document}

\maketitle

\begin{abstract}
Multi-output is essential in machine learning that it might suffer from nonconforming residual distributions, i.e., the multi-output residual distributions are not conforming to the expected distribution. In this paper, we propose "Wrapped Loss Function" to wrap the original loss function to alleviate the problem. This wrapped loss function acts just like the original loss function that its gradient can be used for backpropagation optimization. Empirical evaluations show wrapped loss function has advanced properties of faster convergence, better accuracy, and improving imbalanced data.
\end{abstract}

\section{Introduction}
A deep learning model learns data representations and embeds the knowledge in its edge weights. A loss function is chosen to determine the goodness of training. Currently, backpropagation with gradient descent optimization algorithm is considered as the de facto standard of the learning process. As a deep learning model is learned to optimize its loss function, a loss function can determine the functionality and quality of the learned model, deal with nonconforming data set, and decide the convergence time. In this paper, we propose a wrapped loss function, which can enfold any a regular loss function so that to obtain many excellent properties.

Generally, the task of a deep learning model can be regression or classification. If the task is classification, it means the output variables have class labels; conversely, if the task is regression, the output variable takes continuous vector. The excellent performance of the model depends on an appropriate loss function. For example, the loss function using cross-entropy is more suitable for classification tasks than using least-squares error, because the cross-entropy allows finding a better local optimum for classification in a comparable environment and with randomly initialized weights \cite{pavel2013}. On the other hand, according to Gauss-Markov theorem, the least-squares estimator is the best linear unbiased estimators (BLUE) for linear regression tasks; however, the Gauss-Markov theorem establishes on the assumptions that finite variance of the error is constant for all variables and over time. However, this assumption is too strong for many real-world data. 
 
Besides, multi-outputs is common in machine learning, such as multi-class classification and language translation, which could be trained either by combining several binary classifiers, or all-together training. In all-together deep learning model, loss function has to compute the probability of each class, and the weighted average is a frequently used method to alleviate nonconforming problems.  For instance, if a task deals with imbalanced training data for multi-class classification give higher weights to rare classes can adjust the situation. Besides, least-square error having weights with appropriate distributions keeps the BLUE property even when data are heteroscedastic.  

In this paper, we propose a wrapped loss function to wrap the original loss function by adding weight, estimated by Gauss maximum likelihood function, to each output residual.  This approach can wrap up any loss function to become a "wrapped loss function," and it is the namesake of the proposed approach. The wrapped loss function regularizes the loss of multi-output model outcomes and adjusts gradient accordingly during model training. In the later sections, we give a formal definition of the wrapped loss function and analyze its gradient and other properties.
   
In the empirical evaluations, we build both classification and regression multi-output cases to examine the performance of wrapped loss function, which is applied to the original loss functions. For the first case, we use a fully connected neural network to predict air pollution level, and for the second case, LeNet and AlexNet are used as the fundamental neural network for CIFAR100 test. Besides the accuracy rate, convergence rate, and using imbalanced and/or heteroscedastic data are also used for performance comparisons.

This paper is organized as follows. Section 2 reviews related work. Section 3 gives the problem definition and motivation. Section 4 formally introduces the properties of the wrapped loss function. Section 5 is the empirical evaluation. Finally, Section 6 is the conclusion.

\section{Related Work}
The work of introducing weights in the learning process has inspired this study.

\begin{itemize}
\item Weights on Multi-output Loss Function

The most similar concept with wrapped loss function is cost-sensitive learning and label re-balancing. The target of cost-sensitive learning is to find a minimal cost in the imbalanced label's problems. In general, the errors are from the class with a rare label. In cost-sensitive learning, some methods add weights on loss function to learn its features. For example, there are the Prior Scaling \cite{steve1998} and Minimization of the misclassification costs method \cite{kukar1998}. Alternatively, just pre-set label weights before model training. like Median Frequency Balancing: \cite{bad2015,farabet2012,eigen2015}. ($\alpha_c = mdeian_{freq}/freq(c)$, where freq(c) is the number of data in the class c divided by the total number of data, and $median_{freq}$ is the median of the frequencies.

\item Weights on Data Points

Besides add weights to multi-output loss, add weights to data points has been common practice on controlling outliers rejection. The outliers affect the performance in data mining and machine learning tasks. For example, if a data set includes outliers, the learned model would be misled. The method of Iteratively Reweighted Least Squares(IRLS) \cite{green1984} is used to find the maximum likelihood estimates of a generalized linear model to get the weights of data points to eliminate outliers. 

Another noteworthy modify weights on data usage are AdaBoost \cite{freund1995boosting}. While AdaBoost boosts the performance of a collection of the weak learner, it also modifies the weights of data depending on their current classification performance so that the final classifier achieves the better performance.

\item Weights on Multi-Task Learning

For the multi-task learning, there are two or more related tasks jointly learned, and their loss functions are combined into a loss function. Then the weights can be added to each task to generate a loss function. Huy \cite{huy2017} uses convolutional neural networks and deep neural network coupled with weighted loss function to solve audio event detection, and Sankaran \cite{sankaran2016} shows that the combination of 3 techniques: LSCSR-initialization, Multi-task training, and Class-weighted cross-entropy training gives the best results on keyword spotting. Unlike the above cases use pre-set weights, Kendall \cite{kendall2017} applies IRLS to find maximum likelihood estimates of a multi-task learning model for scene geometry and semantics.
\end{itemize} 

\section{Background and Motivation}
 In this section, we will introduce the problem when dealing with heteroscedasticity by using least square loss function as an example. By $D=(X,Y)$ we denote the given training data, and by $f(x)$ or $f_{W}(x)$ the learned function or neural network with the weight set $W$ of cardinality $d$. Recall $c$ is the number of class or the length of multi-output.
\subsection{Least Square Loss Function}
  The least-square method is famous for finding a curve fitting for a given set of data with assumptions: expected zero error, uncorrelated, and constant variance in the errors. According to Gauss-Markov theorem, if all assumptions are satisfied, the least square error method is BLUE.
\begin{itemize}
\item Least Square Error
\begin{equation} \label{OriginalLoss}
\ell_{original} = \sum_{i=1}^c(y_i-f(x))^2
\end{equation}
\end{itemize}

However, heteroscedasticity is a problem because the original least square method assumes that all residuals are drawn from the same population. This assumption may be violated in some real-world time series and cross-sectional data.
  
This problem can be solved by using a weighted least square estimation to obtain asymptotically efficient estimators. While some approaches add weights to data points, in this paper, we add weights to lose function.

\begin{itemize}
\item Weighted Least Square Error
\begin{equation}\label{HeteroscedasticityLoss}
\ell_{weight}=\sum_{i=1}^co_i(y_i-f(x))^2
\end{equation}
\end{itemize}
where $o_i$'s are the weights for losses.
 \subsection{Gradient Descent of Loss Functions}
The gradient descent of loss function with back propagation is the most popular optimization algorithm for updating the weights of  artificial intelligent network.  In this subsection, we will show the difference of gradient descents between original and weighted losses.
\subsubsection{In The original Loss Function}

Gradient descent is a first-order optimization algorithm to seek for the minimum loss and Eq (3) shows the loss function of i-th output $l_i$ and  its gradient:

\begin{equation}
\begin{split}
\ell_i &= (y_i-f(x))^2\\
\frac{\partial{\ell_i}}{\partial{w}}&=\frac{\partial{\ell_i}}{\partial{f(x)}}\frac{\partial{f(x)}}{\partial{w}}\\
&=2(y_i-f(x))\frac{\partial{f(x)}}{\partial{w}}\\
\end{split}
\end{equation}


The weighted loss function and gradient are defined as below.

\begin{equation}
\begin{split}
\ell_{o_i} &= o_i(y_i-f(x))^2\\
\frac{\partial{\ell_{o_i}}}{\partial{o_i}}&=(y_i-f(x))^2\\
\frac{\partial{\ell_{o_i}}}{\partial{w}}&=\frac{\partial{\ell_i}}{\partial{f(x)}}\frac{\partial{f(x)}}{\partial{w}}\\
&=2o_i\frac{\partial{f(x)}}{\partial{w}}(y_i-f(x))
\end{split}
\end{equation}

  Although the weighted least square error function can alleviate the heteroscedasticity problem, the weights might quickly approach to zero from the partial differentiation of the loss function by the weights. This problem can be solved by introducing a particular loss function that direct the learning process.

\section{Wrapped Loss Function }

In this section, we propose the wrapped loss function and its gradient and generalization error. Formally at first, we define a residual vector by:

\begin{equation} \label{Heteroscedasticity}
Y_i =f(X)+\epsilon_i  \hspace{4mm} where \hspace{2mm} \epsilon_i {\sim} N(0,\sigma_i^2)
\end{equation}
with three assumptions:
\begin{itemize}
  \item Errors are mean zero: E[$\epsilon_i$]=0,
  \item Error are heteroscedastic, that is all have the same finite variance: Var($\epsilon_i$)=$\sigma_i^2$ $<$ $\infty$, 
  \item Distinct error terms are uncorrelated: Cov($\epsilon_i$,$\epsilon_j$)=0, $\forall$i$\neq$j
\end{itemize}

Heteroscedasticity occurs when the error is correlated with an independent variable, for example, in a regression on household saving and income. Low-income people generally save a similar amount of money, while high-income people may have a significant variation in their savings. Also, heteroscedasticity is often the case with cross-sectional or time-series data that it makes estimation inefficient because the actual variance is underestimated. To solve the problem, we propose the wrapped loss function, unlike weighted loss function having zero weight problem in its gradient, the wrapped loss function is differentiable and learnable. 


\begin{lemma}\label{NegLogLikelihood}
If Eq. \ref{Heteroscedasticity} holds, the negative log-likelihood is decided by
$$\ell(\mathbf{\sigma},f)=\sum_{i=1}^c\left(\frac{\ell_i}{\sigma_i^2}+\log(\sigma_i^2)\right)$$
Furthermore, we have
$$\aligned
\frac{\partial{\ell(\mathbf{\sigma},f)}}{\partial{\sigma_i}}&=\frac{2}{\sigma_i}-\frac{2(y_i-f(x))^2}{\sigma_i^3}\\
\frac{\partial{\ell(\mathbf{\sigma},f)}}{\partial{w_j}}&=\sum_{i=1}^c \frac{1}{\sigma_i^2}\frac{\partial{\ell_i}}{\partial{w_j}}
\endaligned
$$ 
\end{lemma}
\begin{proof}
From Eq. \ref{Heteroscedasticity} , 
$$
p(y_i|x,\sigma_i)=\frac{1}{\sqrt{2\pi\sigma_i^2}}exp\left(\frac{-(y_i-f(x))^2}{2\sigma_i^2}\right)
$$
For the multi-output model, it specifies the joint density function for all outputs by the method of maximum likelihood function, derived as:
$$ 
\aligned 
L(\mathbf{\sigma}|\mathbf{Y}) 
&=\prod_{i=1}^c p(y_i|x,\sigma_i)\\
\endaligned 
$$ 
where $\mathbf{\sigma}=(\sigma_1,\dots,\sigma_c)$ and $\mathbf{Y}=(Y_1,...,Y_c)$. The negative log-likelihood is as below:
$$ 
\aligned  
-&\log{L(\mathbf{\sigma}|\mathbf{Y})}\\
=&\frac{c}{2}\log(2\pi)+\sum_{i=1}^c\frac{log(\sigma_i^2)}{2}+\sum_{i=1}^c\frac{(y_i-f(x))^2}{2\sigma_i^2}\\
\endaligned 
$$  
Recall that the neural network $f(x)$ depends on the parameters $w_j$ and $\ell_i=(y_i-f(x))^2$. Hence we denote $\sum_{i=1}^c\left(\frac{\ell_i}{\sigma_i^2}+\log(\sigma_i^2)\right)$ as $\ell(\mathbf{\sigma},f)$. The partial derivatives of $\ell(\mathbf{\sigma},f)$ are easy to calculate.
\end{proof}

Comparing Eq.\ref{HeteroscedasticityLoss} with the loss in Lemma \ref{NegLogLikelihood}:
$$\aligned 
\ell_{weight}=& \sum_{i=1}^c o_i\ell_i \\
\ell(\mathbf{\sigma},f)=&\sum_{i=1}^c\left(\frac{\ell_i}{\sigma_i^2}+\log(\sigma_i^2)\right)
\endaligned
$$

We found that it is natural to consider $\sigma_i^{-2}$ as an estimator of $o_i$ and the term $\log(\sigma_i^2)$ as an regularizer. Hence, we heuristically assume $\forall i\in [c], o_i\approx\frac{k_i}{\sigma_i^2}$, where $k_i$'s are constants.  
\begin{definition}(The Wrapped Loss Function)
$$\ell_{wrap} = \sum_{i=1}^c\left(o_i\ell_i +{\log o_i^{-1}}\right).$$
\end{definition}
The function is called the "Wrapped Loss Function" since it can wrap up any loss function to become a new function. This wrapped loss function is differentiable that it can derive the weights $o_i$ to minimize the loss function.

The following Lemma and Theorem show the core idea in our training process, which combines similar concepts of 
the EM algorithm, (output loss) normalization, and regularization.
\begin{lemma} With the same assumptions (Eq.\ref{Heteroscedasticity}), 
\begin{enumerate}
\item The estimator $\hat{\sigma_i}^2=(y_i-f(x))^2$ maximizes the likelihood of Lemma \ref{NegLogLikelihood}.
\item If $o_i\leftarrow \hat{\sigma}^{-2}$, then $\frac{\partial{\ell_{wrap}}}{\partial{o_i}}=\ell_i-\frac{1}{o_i}$ and $\frac{\partial{\ell_{wrap}}}{\partial{w_j}}=\sum_{i=1}^{c}{o_i\frac{\partial{\ell_i}}{\partial{w_j}}}$.
\end{enumerate}
\end{lemma}
\begin{proof}
By lemma \ref{NegLogLikelihood},
$ \frac{\partial{\ell(\mathbf{\sigma},f)}}{\partial{\sigma_i}}=0 
\Rightarrow \sigma_i^2=(y_i-f(x))^2$.
This means for the given training data, setting the the residual sum of squares $\hat{\sigma_i}^2$ as above can minimize the loss $\ell(\mathbf{\sigma},f)$, which maximizes the likelihood function in Lemma \ref{NegLogLikelihood}. 


The second item is immediate by calculation. 
\end{proof}
It is easy to see that if $\forall i\in[c]$, $o_i=1$ then $\ell_{wrap}=\ell_{original}$. The following shows that when $o_i$'s are all near 1, the proposed $\ell_{wrap}$ approximates $\ell_{original}$ well.
\begin{theorem} \label{Thm1}
For given data $D$, class number $c$ and neural network $f_{W}(x)$, let $L=max_{i\in[c],x\in X}\{(y_i-f_{W}(x))^2\}$. Assume $L>1$ and $\delta\in (0,(c(L+1))^{-2}]$. If $\forall i\in [c], |o_i-1|\leq \delta$, then $|\ell_{wrap}-\ell_{original}|\leq \frac{1}{c(L+1)}$. 
\end{theorem}
\begin{proof}
Since $\log(1-\delta)=\delta-O(\delta^2)$,
$$
\aligned
\ell_{wrap}-\ell_{original}
&=\sum_{i=1}^{c}{(o_i-1)\ell_i+\log(o_i^{-1})}\\
&\leq \sum_{i=1}^{c}{\delta L+\log(o_i^{-1})}\\
&\leq c(\delta L+\delta -O(\delta^2))\\
&\leq c\delta (L+1)\\
&\leq \frac{1}{c(L+1)}.
\endaligned
$$
\end{proof}
Theorem \ref{Thm1} implies the neural network $f(x)$ trained by the wrap loss $\ell_{wrap}$, if $o_i$'s are all near $1$, can be a approximation of the problem with the original loss Eq. \ref{OriginalLoss}.

\begin{algorithm}[h] 
\caption{Wrapped Loss Based Training}
\label{alg::WrappedLossGradientUpdate}
\begin{algorithmic}[1]
\REQUIRE
wrapped loss function $\ell_w$;
data input $x$;
model parameters $w$;
wrapped parameters $o$;
number of epoch $t$;
learning rate $\alpha$;
\STATE initial $o=1$ and $w$ and $t=0$;
\REPEAT
\STATE compute  $g_{o_i}\leftarrow\frac{\partial{\ell_{wrap}}}{\partial{o_i}}$ 
\STATE $\forall i\in \lbrace1,...,c\rbrace,o_{i,t+1}\leftarrow o_{i,t} -\alpha g_{o_i}$
\STATE compute $g_{w_j}\leftarrow\frac{\partial{\ell_{wrap}}}{\partial{w_j}}$ 
\STATE $\forall j\in \lbrace1,...,d\rbrace, w_{j,t+1}\leftarrow w_{j,t}-\alpha g_{w_j}$
\STATE $t\leftarrow t+1$
\UNTIL{converge}
\end{algorithmic}
\end{algorithm}
Note that Algorithm \ref{alg::WrappedLossGradientUpdate} is used for training only. The real output is the original loss.

\subsection{Properties of Wrap Error}
In the first glance at Wrap loss function $\ell_{wrap}$, it is not easy to see how $\ell_{wrap}$ effects the training. The following shows some advantages. 

\begin{lemma} \label{StepSize}
The setting of $o_i$ in Line 4 of Algorithm \ref{alg::WrappedLossGradientUpdate} helps the convergence.
\end{lemma}
\begin{proof}
Observe that for $i\in[c]$,
\begin{itemize}
\item if $\ell_{i,t}<\ell_{i,t+1}$, i.e. $\frac{1}{\ell_{i,t}}>\frac{1}{\ell_{i,t+1}}$; this means the updating at time $t$ was in a bad direction (since the loss increases), so at time $t+1$, the updating should be  cautious with a smaller ratio $o_{i,t+1}\leftarrow \frac{1}{\ell_{i,t+1}}$. 
\item if $\ell_{i,t}=\ell_{i,t+1}$, then ${o_i\frac{\partial{\ell_i}}{\partial{w_i}}}=0$. Furthermore, if for all $j\in[d]$ we have $\sum_{i=1}^{c}{o_i\frac{\partial{\ell_i}}{\partial{w_j}}}=0$, then Algorithm \ref{alg::WrappedLossGradientUpdate} convergences.
\item if $\ell_{i,t}>\ell_{i,t+1}$, i.e. $\frac{1}{\ell_{i,t}}<\frac{1}{\ell_{i,t+1}}$; this means the updating at time $t$ was in a good direction (since the loss decreases), so at time $t+1$, the updating can be  vigorous with a larger ratio $o_{i,t+1}\leftarrow\frac{1}{\ell_{i,t+1}}$. 
\end{itemize}
\end{proof}

Now we can further discuss loss surface with wrapped; the wrapped loss surfaces displayed how the loss affected by the parameter $o_i$. The wrap loss $\ell_{wrap} = \sum_{i=1}^c\left(o_i\ell_i +{\log o_i^{-1}}\right)$ is complicated; to illustrate the main idea, we just consider one term in the summation and simplify as $WrapErr=oP^2+\log(o^{-1})$, where $P$ is the square error of the prediction and $o$ is the parameter in the wrap loss.  

For instance, we assume predict target is 20, and the loss function is the least square. The wrapped loss has a parameter $o_i$ to control each outputs loss. As shown in Figure \ref{2DWrapSurface}, if the higher $o_i$, then the loss will be a sharper parabola; on the contrary, if the lower $o_i$, then the loss will be a flatter curve. Besides, if the  $o_i$ equals to 1, the wrapped loss is the same with the original least square loss, and less than 1 will be more flat than original.

\begin{figure}[hbtp]
\includegraphics[scale=0.5]{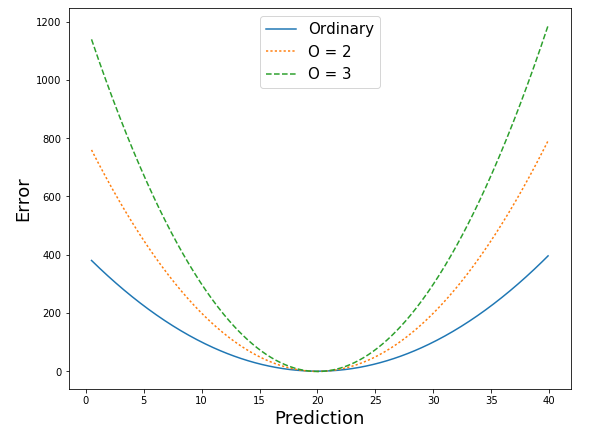}
\caption{Illustration of Wrap Error Surface 2D} 
\label{2DWrapSurface}
\end{figure}

In Figure \ref{3DWrapSurface}, we have a global view of the illustrated wrapped loss surface. As mentioned above, the loss and the gradient affected by the parameter $o_i$, and the updating step seems "jump" to the next location.

\begin{figure}[hbtp]
\includegraphics[scale=0.7]{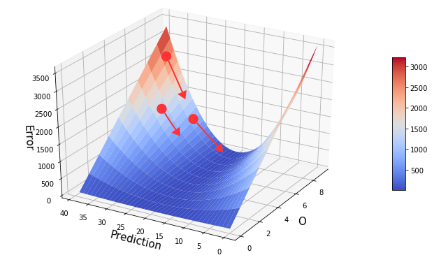}
\caption{Illustration of Error Surface 3D} 
\label{3DWrapSurface}
\end{figure}

\begin{lemma} \label{HatSigma}
If $\hat{\sigma}^2=(y_i-f(x))^2=\ell_i$, then  $E\left[\frac{\hat{\sigma_i}^2}{\sigma_i^2} \right]=DoF_i,$ where $DoF_i$ is the degree of freedom of the corresponding model.
\end{lemma}
\begin{proof}
Since $\hat{\sigma_i}^2=(y_i-f(x))^2=\ell_i$ and $y_i-f(x)=\epsilon_i {\sim} N(0,\sigma_i^2)$, it is well known in the chi-square test for variance that $\frac{o_i^2}{\sigma_i^2}\sim\chi_{DoF_i}^2$. Hence $E\left[\frac{\hat{\sigma_i}^2}{\sigma_i^2} \right]=DoF_i$. 
\end{proof}

In general, the degree of freedom can be a measure of the complexity of a statistical learning model.

\begin{theorem} \label{Thm2}
Suppose Eq. \ref{Heteroscedasticity}, and the degree of freedom of the neural network $f(x)$ is fixed for multi-output, i.e. $\forall i\in[c]$, $DoF_i=DoF$. 
Then, on average and approximately, wrap loss function can be presented as 
$$
E[\ell_{wrap}|ALGO\ref{alg::WrappedLossGradientUpdate} ]
\approx c(1+\log(DoF))+\sum_{i=1}^{c}{\log(E[\sigma_i^2])}.
$$
\end{theorem}
\begin{proof}
Lemma \ref{HatSigma} implies that on average $\hat{\sigma_i}^2$ equals to $DoF_i\cdot\sigma_i^2$ and by assumption $DoF_i=DoF$, which means $\hat{\sigma_i}^2\approx DoF\cdot\sigma_i^2$. On the other hand, in Algorithm \ref{alg::WrappedLossGradientUpdate}, we assign the residual sum of squares $\hat{\sigma_i}^2$, i.e. $\ell_i=(y_i-f(x))^2$, to $o_i^{-1}$. Hence if the expectation conditions on Algorithm \ref{alg::WrappedLossGradientUpdate}, then
$$\aligned
\sum_{i=1}^{c}{o_i\ell_i+\log(o_i^{-1})}
&= \sum_{i=1}^{c}{\ell_i^{-1}\ell_i+\log(\hat{\sigma_i}^2)}\\
&\approx\sum_{i=1}^{c}{1+\log(DoF\cdot{\sigma_i}^2)}\\
&=c+c\log(DoF)+\sum_{i=1}^{c}{\log(\sigma_i^2)}.
\endaligned
$$
\end{proof}

\begin{figure}[hbtp]
\includegraphics[scale=0.5]{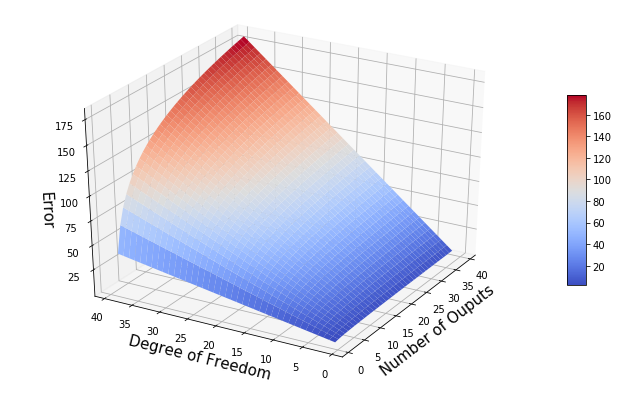}
\caption{Illustration of Theorem \ref{Thm2} }
\label{3DWrapGE}
\end{figure}

Figure \ref{3DWrapGE} shows the simplified surface of Theorem \ref{Thm2}. The term $\sum_{i=1}^clog(\sigma_i)$ is omitted to present the relation between the degree of freedom and the number of outputs. From the degree of freedom, the approximation surface is a logarithmic growth function when the number of outputs is fixed. On the contrary, the illustrated surface is a linear growth function of the number of outputs when a fixed degree of freedom. 

\section{Empirical Evaluations}
In this section, we prepare two sets of experiments with distinct characteristics to examine the properties of the proposed wrapped loss function.  The first set is fine atmospheric particulate matter (PM2.5) forecast, which is formulated as a multi-output regression.  The second set is the image object recognition, which is a multi-class classification task. These show the wrapped loss functions has the following properties from experiments. 

\begin{itemize}
  \item Wrapped loss function improves the accuracy of the other models with the original loss function.
   \item Wrapped loss function converges faster than original loss function.
  \item Wrapped loss function alleviates imbalanced data problem.
\end{itemize}

\subsection{PM2.5 Prediction}

\subsubsection{Data Description}
Fine Atmospheric particulate matter, known as PM2.5, is a collection particulates matter with a  diameter of 2.5 $\mu m$ or less that they have significant impacts on environment and climate, and damages to human health. In this task, we predict PM2.5 in an urban area (Taipei, Taiwan, in this study) using open data provided by the Environmental Protection Administration.  The data consists of 18 monitor stations with features including  weather features (precipitation, pressure, temperature, wind speed, wind direction, etc.),  air quality features (PM2.5, CO, NO2 , PM10, RH, SO2, etc.), and  traffic-related features (average vehicle speed, traffic load, intersection numbers, etc.). The data of the year of 2015 is used as the training set and 2016 as the test set. Note that the PM2.5 level is heteroscedastic since the variances are different at a different time of day. For instance, at 3 am, the PM2.5 level is low and has a small variance since there is almost no traffic on the roads, while at 2 pm, the variance becomes large depending on the events of the city and traffic situation.   
   
\subsubsection{Model Architecture}
In this set of experiments, we design a neural network with 6 fully connected layers, namely, FC1024, FC512, FC256, FC128, FC64 and output layer (Figure \ref{DNNforPM25}), where there is a dropout layer or batch normalization between each layer. The input vector is the information of a day and predicts target is a day after hours of PM2.5.

\begin{figure}[htp]
\includegraphics[scale=0.55]{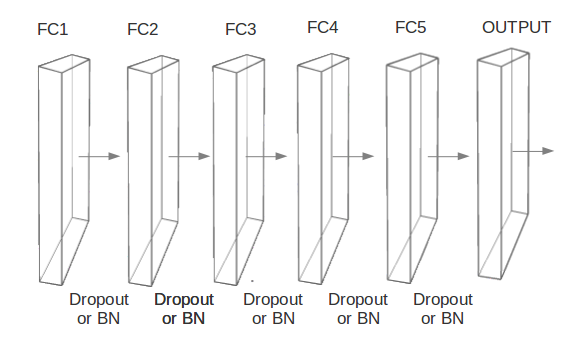}
\caption{Deep Learning Model} 
\label{DNNforPM25}
\end{figure}

The optimization algorithm of this model is AdaGrad \cite{john2011} with a fixed learning rate is set as 0.01 and mini-batch size is 10. Mean square error is the loss function and.

\subsubsection{Data Preprocessing}

Generally, deep learning method automatically extracts useful knowledge to create a model without handcrafted feature engineering. In this experiment, we only normalize training data to Gaussian distributions with zero mean and unit variance.

\subsubsection{Results}

The goal of the experiment is to predict the PM2.5 level of BangQiao and Cailiao in 6 hours, 12 hours, 18 hours and 24 hours with dropout (DO) \cite{srivastava2014}, batch normalization (BN) \cite{sergey2015} and wrapped loss function (WR) and their combinations (WR+DO, WR+BN, and WR+BN+DO). Note that this task is a heteroscedasticity problem. For example,  PM2.5 at 2 pm and 3 am has a different variance. The RMSE comparison between the original model with others techniques and the one with wrapped loss function are displayed in Table \ref{Table01} for and the Banqiao monitoring station, and Table \ref{Table02} for Cailiao monitoring station. 

The results of the experiment show that if the training uses only dropout, batch normalization, or wrapped loss, then the proposed wrapped loss is usually the second-best. 
Although the wrapped loss function performs worse than the original model with dropout, 
the wrapped loss function works with dropout and/or batch normalization at the same time, and it shows the performance is best for wrapped loss combined with both dropout and batch normalization. Generally, the wrapped loss function usually improves the performance of the original methods.

\begin{table}[!h] 
\centering  
    \caption{DNN: PM2.5 Best RMSE of Banqiao Station}     
\label{Table01} 

\begin{small}  
\begin{tabular}{|c|c|c|c|c|c|c|}  

  \hline  
  \bfseries  \#Times    &\bfseries DO  & \bfseries BN  &\bfseries WR &\bfseries W+D & \bfseries W+B\\
 \hline  
 6     & 11.06  &13.05  &13.2   &11.08  &13.42  \\  
 \hline  
 12    & 12.83  &14.77  &14.14  &12.19  &15.97 \\  
  \hline  
 18    & 12.96  &15.64  &15.62  &12.65  &15.07 \\ 
 \hline  
 24    & 13.66  &18.29  &16.94  &13.30  &15.36 \\  
 \hline  
 \end{tabular}  
    \end{small}  
\end{table} 

\begin{table}[!h]
\centering  
    \caption{DNN: PM2.5 Best RMSE of Cailiao Station}     
\label{Table02}    
\begin{small}  
\begin{tabular}{|c|c|c|c|c|c|c|}  
  \hline  
  \bfseries  \#Times    &\bfseries DO  & \bfseries BN  &\bfseries WR &\bfseries W+D & \bfseries W+B \\  
 \hline  
 6     & 8.83  &14.42  &13.48   &9.13  &13.18 \\  
 \hline  
 12    & 9.65  &14.03  &13.73   &10.37  &12.13 \\  
  \hline  
 18    & 10.29  &15.02  &12.63  &11.02  &13.95 \\ 
 \hline  
 24    & 10.93  &14.1  &12.74   &10.56  &15.18\\  
 \hline  
 \end{tabular}  
    \end{small}  
\end{table}

We also compared the number of Epochs needed to complete the training process. According to Table \ref{Table03}, it can be seen that the wrapped loss function has fastest convergence rate. 

\begin{table}[!h]
\centering  
    \caption{DNN: Epoch of Best PM2.5 RMSE of Banqiao Station}     
 \label{Table03}    
\begin{small}  
\begin{tabular}{|c|c|c|c|c|c|c|}  

  \hline  
  \bfseries  \#Times    &\bfseries DO  & \bfseries BN  &\bfseries WR &\bfseries WR+DO & \bfseries WR+BN \\  
 \hline  
 6     & 1789  &1  &1   &1665   &1   \\  
 \hline  
 12    & 1262  &1  &1   &1190   &1 \\  
  \hline  
 18    & 1427  &2  &1   &1163   &1 \\ 
 \hline  
 24    & 641  &1  &1    &1294   &1 \\  
 \hline  
 \end{tabular}  
    \end{small}  
\end{table}

\begin{table}[!h] 
\centering  
    \caption{DNN: Epoch of Best PM2.5 RMSE of Cailiao Station}     
 \label{Table04}    
\begin{small}  
\begin{tabular}{|c|c|c|c|c|c|c|}  

  \hline  
  \bfseries  \#Times    &\bfseries DO  & \bfseries BN  &\bfseries WR &\bfseries WR+DO & \bfseries WR+BN \\  
 \hline  
 6     & 1769  &3  &2   &953   &2 \\  
 \hline  
 12    & 1503 &1  &1   &802   &1  \\  
  \hline  
 18    & 1297  &2  &1   &986   &1\\ 
 \hline  
 24    & 1132  &1  &1   &787   &1 \\  
 \hline  
 \end{tabular}  
    \end{small}  
\end{table} 

\subsection{Image Object Recognition}
\subsubsection{Data Description}
Different from the multi-output regression in the previous subsection, the second task is of multi-class classification to assign a class to the image object. The CIFAR-100 \cite{krizhevsky2009learning} data set consists of 60000 32$\times$32 RGB color images in 100 classes. In each class, there are 600 images with 500 training data and 100 testing data. In addition, the 100 classes are grouped into 20 super-classes. In this task, it is to classify color photographs of object into one of 100 classes.

\subsubsection{Model Architecture}
Convolution neural networks (CNN) is one of most popular and successful deep learning models \cite{lecun1998} in image processing tasks. We use  LeNet-5 and AlexNet \cite{alex2012} as the original functions for image recognition task. In the architecture of LeNet (as shown in the left part of Figure \ref{LeNet_AlexNet}), after two stacks of convolution with padding and max pooling sections, the activation function ReLU is computed. Then, two fully connected layers follow, and a softmax layer is at the end. The loss function is cross entropy function. The reason we choose LeNet and AlexNet is that their source codes were available from tensorflow and have credibility. The training parameters are set as: Fixed learning rate of 0.001, Mini-batch size of 100. And the AlexNet (as shown in the right part of Figure \ref{LeNet_AlexNet}) had a similar architecture as LeNet but is deeper and has more filters.

\begin{figure}[htp] 
\includegraphics[scale=0.5]{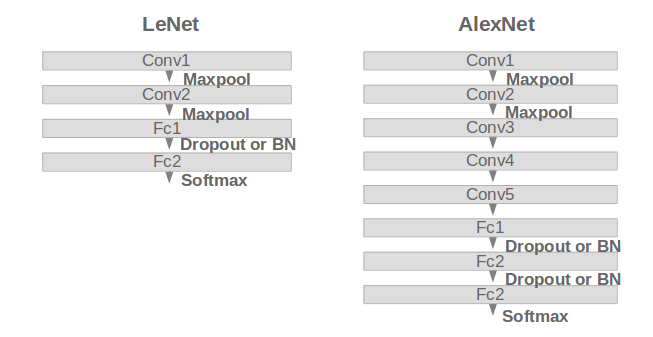}
\caption{LeNet and AlexNet Model} 
\label{LeNet_AlexNet}
\end{figure}

\subsubsection{Results}
First, We run both LeNet and AlexNet to classify pictures into one of 5, 10, 30, 50, 80, 100 classes. The results are shown in Table \ref{Table05} and Table \ref{Table06}. It can be observed that the wrapped loss function performs better than DO and BN in both LeNet and AlexNet, especially when the task is classified image objects into large number of classes.

\begin{table}[!h] 
\centering  
    \caption{LeNet: Best Accuracy of CIFAR 100 }     
 \label{Table05}    
\begin{small}  
\begin{tabular}{|c|c|c|c|c|c|c|}  

  \hline  
  \bfseries  \#Classes    &\bfseries DO  & \bfseries BN  &\bfseries WR &\bfseries WR+DO & \bfseries WR+BN\\  
 \hline  
 5     & 75.6     &77.4  &78   &77.4  & 77.2 \\  
 \hline  
 10    & 66.2     &66.5  &68.2 &65.7  & 69.1 \\  
  \hline  
 30    & 46.7     &48.3  &50.7 &49.8  &50.7\\ 
 \hline  
 50    & 39.1     &40.1  &42.4 &42.7  &42.4\\  
 \hline  
 80    & 36.2     &37.6  &38.8 &38.3 &38.2 \\  
 \hline  
 100   & 32.6     &34    &35.9 &35.8  &36\\  
 \hline  
 \end{tabular}  
    \end{small}  
\end{table} 

\begin{table}[!h] 
\centering  
    \caption{AlexNet: Best Accuracy of CIFAR 100 }     
 \label{Table06}    
\begin{small}  
\begin{tabular}{|c|c|c|c|c|c|c|}  

  \hline  
  \bfseries  \#Classes    &\bfseries DO  & \bfseries BN  &\bfseries WR &\bfseries WR+DO & \bfseries WR+BN\\  
 \hline  
 5     & 78.4  &79.8  &77.4   &78.4  &78.8 \\  
 \hline  
 10    & 66    &67    &67.5   &66    &66.5 \\  
  \hline  
 30    & 40.7  &40.5  &44.4   &40.7  &46.3\\ 
 \hline  
 50    & 31.8  &31.4  &37.5   &31.8  &36.8\\  
 \hline  
 80    & 27.6  &27.9  &33     &27.6  &33.2 \\  
 \hline  
 100   & 25.6  &25.2  &30.3   &25.6  &30.1\\  
 \hline  
 \end{tabular}  
    \end{small}  
\end{table} 

The numbers of Epochs needed to complete the training process are compared. From the results of Table \ref{Table07} and Table \ref{Table08}, it can be seen that the wrapped loss function converged mush faster than the LeNet and AlexNet. It is also interesting to find out when the wrapped loss function works with dropout or batch normalization, the training process is completed in an even shorter time. The possible reason is that the wrapped loss function approach adopts idea similar to Gauss-Newton method \cite{wedderburn1974quasi} that fast convergence is a famous advantage of Newton's optimization methods. 

\begin{table}[!h]  
\centering  
    \caption{LeNet: Epoch of Best Accuracy of CIFAR 100}     
\label{Table07}   
\begin{small}  
\begin{tabular}{|c|c|c|c|c|c|c|}  

  \hline  
  \bfseries  \#Classes    &\bfseries DO  & \bfseries BN  &\bfseries WR &\bfseries WR+DO & \bfseries WR+BN\\  
 \hline  
 5     & 9k   &26k   &8k    &3k     &5k \\  
 \hline  
 10    & 35k  &29k   &13k   &11k    &14k \\  
  \hline  
 30    & 39k  &42k   &22k   &18k    &24k\\ 
 \hline  
 50    & 57k  &65k   &22k   &20k    &29k\\  
 \hline  
 80    & 76k  &140k  &56k   &21k    &55k\\  
 \hline  
 100   & 91k  &104k  &56k   &28k    &9k\\  
 \hline  
 \end{tabular}  
    \end{small}  
\end{table} 

\begin{table}[!h]  
\centering  
    \caption{AlexNet: Epoch of Best Accuracy of CIFAR 100}     
\label{Table08}    
\begin{small}  
\begin{tabular}{|c|c|c|c|c|c|c|}  

  \hline  
  \bfseries  \#Classes    &\bfseries DO  & \bfseries BN  &\bfseries WR &\bfseries WR+DO & \bfseries WR+BN\\  
 \hline  
 5     & 29k  &30k   &4k    &15k    &5k \\  
 \hline  
 10    & 50k  &45k   &7k    &5k     &6k \\  
  \hline  
 30    & 36k  &50k   &10k   &13k    &14k\\ 
 \hline  
 50    & 39k  &46k   &16k   &15k    &17k\\  
 \hline  
 80    & 58k  &53k   &18k   &20k    &24k\\  
 \hline  
 100   & 59k  &62k   &26k   &24k    &27k\\  
 \hline  
 \end{tabular}  
    \end{small}  
\end{table}


\subsubsection{imbalancing data}

Here we show the capability of wrapped loss function in improving the performance on imbalanced data set. 10 classes from CIFAR100 are selected randomly and the number of data a selected class is reduced  to 2\%, 5\%, 10\%, 20\% and 30\% respectively of the original size. For example, the adjusted class only has 50 images when the proportion is down to 10\%. We compare the performances on accuracy of three settings: LeNet, LeNet with wrapped loss function and 
the frequency balancing method \cite{bad2015,farabet2012,eigen2015}. 
 
The experiment results of using imbalanced data as in Table \ref{Table09} shows the wrapped loss function and the frequency balancing has similar accuracy rate, while both are better than the original LeNet. Table \ref{Table09} has an ADJ column shows the accuracy rate of the adjusted class that the Frequency Balancing method has a slightly better performance than the wrapped loss function. However, the wrapped loss function still has the shortest time to complete the training process, as in Table \ref{Table10}. The possible reason is if the class has larger total loss, $o_i$ will be smaller. So the rare class get the higher gradient that it affects the performance on imbalanced data.

\begin{table}[!h]  
\centering  
    \caption{LeNet: Accuracy of Imbalanced Data}
\label{Table09}    
\begin{small}  
\begin{tabular}{|c|c|c|c|c|c|c|}  
 \hline 
 &\multicolumn{2} {c|}{\bfseries Original} &\multicolumn{2} {c|}{\bfseries Freq Balancing} &\multicolumn{2} {c|}{\bfseries Wrapped Loss}\\
  \hline  
  \%     & ADJ &Total & ADJ & Total & ADJ & Total \\  
 \hline  
 2\%    & 3    & 60.4 & 16   & 60.4 &18     &61.7\\  
 \hline  
 5\%    & 25   & 61.2 & 35  & 62.5  &30    &63.6\\  
 \hline   
 10\%   & 39   & 63.2 & 51  & 64.5  &45    &64.4 \\  
 \hline  
 20\%   & 68   & 65.3 & 74  & 65  &68    &65.7\\  
 \hline  
 30\%   & 74   & 66.2 & 76  & 66.2  &73    &68.1\\  
 \hline  
 \end{tabular}  
    \end{small}  
\end{table}

\begin{table}[!h] 
\centering  
    \caption{LeNet: Epoch of Best Accuracy of Imbalanced Cifar 100 with different Imbalanced Rate}     
\label{Table10}     
\begin{small}  
\begin{tabular}{|c|c|c|c|c|}  

  \hline  
    \%    & {\bfseries Original}  & {\bfseries Freq Balancing}    & {\bfseries Wrapped Loss}\\  
 \hline  
 2\%     & 34k     & 26k   &11k  \\  
 \hline  
 5\%     & 37k     & 38k   &12k  \\  
  \hline  
 10\%    & 48k     & 37k   &10k  \\ 
 \hline  
 20\%    & 54k     & 29k  &7k  \\  
 \hline  
 30\%    & 35k     & 30k  &14k  \\  
 \hline   
 \end{tabular}  
    \end{small}  
\end{table}

\section{Conclusions}

We have introduced the approach of wrapped loss function and shown it alleviates the problem of the heteroscedastic variances, it has fast convergence rate and improves accuracy for both multi-output regression task and multi-class classification on deep learning model.

In theoretical aspects, we analyze the wrapped loss function and derive its gradient, which leads to the training process in Algorithm 1, and obtain an approximation of expected error. The training process using wrapped loss iteratively modulates the updating ratio according to the corresponding residual loss, which boosts the convergence rate. The approximation of expected wrap error which conditions on the execution of our training process shows the effects of the length of multi-output (or the number of class) and the model complexity (measured by the degree of freedom).

In the empirical evaluations, we compared original loss and wrapped loss on four experiments of multi-output regression and classification tasks. The experiment results are encouraging, which display wrapped loss can have better accuracy rate than batch normalization and dropout in the CIFAR100 case. In the air pollution case, if we combined wrapped loss function with batch normalization and dropout, it improves the performance too. Moreover, it has advantageous properties of faster convergence rate and alleviation the ill effect of imbalance labels. We believe the wrapped function is convenient to apply and its gain cannot be overlooked during the training phrase of deep learning model.

\bibliography{example_paper.bib}
\bibliographystyle{plain}
\end{document}